\documentclass[11pt]{article}

\usepackage{fullpage}
\usepackage{graphicx} 

\usepackage{times,url,bm}
\usepackage[normalem]{ulem}
\usepackage{amsthm,amsmath,amssymb,mathtools,epsfig,color,float,graphicx,verbatim}
\usepackage{algorithm,algorithmic}
\usepackage{bbm}
\usepackage[square,numbers]{natbib}
\usepackage{epstopdf}
\usepackage{tikz}
\usetikzlibrary{shapes.geometric}
\usetikzlibrary{positioning}
\usepackage{pgfplots}
\pgfplotsset{compat=1.18}
\usepackage{authblk}
\usepackage{comment}
\usepackage{booktabs}
\usepackage{caption}
\usepackage{subcaption}

\usepackage{hyperref}

\hypersetup{
	colorlinks   = true, 
	urlcolor     = blue, 
	linkcolor    = blue, 
	citecolor   = black 
}

\newtheorem{theorem}{Theorem}[section]
\newtheorem{proposition}[theorem]{Proposition}
\newtheorem{lemma}[theorem]{Lemma}
\newtheorem{corollary}[theorem]{Corollary}

\newcommand{\reals}{\mathbb{R}}

\newcommand{\sign}{\mathrm{sign}}

\newcommand{\relu}[1]{\left[ #1 \right]_+}

\newcommand{\bx}{\mathbf{x}}

\newcommand{\bw}{\mathbf{w}}

\newcommand{\bzero}{\mathbf{0}}

\newcommand{\Ocal}{\mathcal{O}}

\newcommand{\Ncal}{\mathcal{N}}

\newcommand{\Ucal}{\mathcal{U}}

\newcommand{\norm}[1]{\left\|#1\right\|}
\newcommand{\inner}[1]{\left\langle#1\right\rangle}
\newcommand{\p}[1]{\left(#1\right)}

\newcommand{\abs}[1]{\left|#1\right|}
\newcommand{\ceil}[1]{\left\lceil#1\right\rceil}
\newcommand{\floor}[1]{\left\lfloor#1\right\rfloor}

\newcommand{\poly}{\mathrm{poly}}

\newcommand{\Max}{\operatorname{Max}}

\newcommand{\subsecref}[1]{Subsection~\ref{#1}}

\renewcommand{\eqref}[1]{Eq.~(\ref{#1})}
\newcommand{\lemref}[1]{Lemma~\ref{#1}}
\newcommand{\corollaryref}[1]{Corollary~\ref{#1}}

\newcommand{\propref}[1]{Proposition~\ref{#1}}
\newcommand{\appref}[1]{Appendix~\ref{#1}}

\newcommand{\itemref}[1]{Item~\ref{#1}}

\title{A Depth Hierarchy for Computing the Maximum\\in ReLU Networks via Extremal Graph Theory}

\author[1]{Itay Safran}
\affil[1]{Ben-Gurion University}

\date{}

\tikzset{
	neuron/.style={circle, draw=black, fill=black!5, very thick, minimum size=15mm},
	linearactivation/.pic={
		\draw[thick,->] (-0.6, 0) -- (0.6, 0);
		\draw[thick,->] (0, -0.6) -- (0, 0.6);
		\draw[ultra thick,blue,-] (-0.45, -0.45) -- (0.45, 0.45);
	},
	reluactivation/.pic={
		\draw[thick,->] (-0.6, 0) -- (0.6, 0);
		\draw[thick,->] (0, -0.6) -- (0, 0.6);
		\draw[ultra thick,blue,-] (-0.45, 0) -- (0, 0);
		\draw[ultra thick,blue,-] (0, 0) -- (0.45, 0.45);
	},
	creluactivation/.pic={
		\draw[thick,->] (-0.6, 0) -- (0.6, 0);
		\draw[thick,->] (0, -0.6) -- (0, 0.6);
		\draw[ultra thick,blue,-] (-0.6, 0) -- (0, 0);
		\draw[ultra thick,blue,-] (0, 0) -- (0.31, 0.31);
		\draw[ultra thick,blue,-] (0.29, 0.3) -- (0.5, 0.3);
	},
	complexnode/.pic={
		\draw (0,0) ellipse (1 and 1.5)
		(0,2.5) circle (1)
		(2.5,0) circle (1.5);},
	declare function={erf(\x)=(1-1/(1+0.278393*\x+0.230389*\x*\x+0.000972*\x*\x*\x*+0.078108*\x*\x*\x*\x)^4);},
	erfactivation/.pic={
		\draw[thick,->] (-0.6, 0) -- (0.6, 0);
		\draw[thick,->] (0, -0.6) -- (0, 0.6);
		\draw[ultra thick, domain=0:0.5, smooth, variable=\x, blue] plot ({\x}, {erf(7*\x)/4});
		\draw[ultra thick, domain=-0.5:0, smooth, variable=\x, blue] plot ({\x}, {-erf(7*-\x)/4});},
	thresholdactivation/.pic={
		\draw[thick,->] (-0.6, 0) -- (0.6, 0);
		\draw[thick,->] (0, -0.6) -- (0, 0.6);
		\draw[ultra thick,blue,-] (-0.5, 0) -- (0, 0);
		\draw[ultra thick,blue,-] (0, 0.3) -- (0.5, 0.3);
	},
    highlightedrelu1/.pic={
        \draw[thick,->] (-0.6, 0) -- (0.6, 0);
        \draw[thick,->] (0, -0.6) -- (0, 0.6);
        \draw[ultra thick,blue,-] (-0.45, 0) -- (0, 0);
        \draw[ultra thick,blue,-] (0, 0) -- (0.45, 0.45);
        \draw[ultra thick,red,-] (-0.6, 0) -- (-0.1, 0);
    },
    highlightedrelu2/.pic={
        \draw[thick,->] (-0.6, 0) -- (0.6, 0);
        \draw[thick,->] (0, -0.6) -- (0, 0.6);
        \draw[ultra thick,blue,-] (-0.45, 0) -- (0, 0);
        \draw[ultra thick,blue,-] (0, 0) -- (0.45, 0.45);
        \draw[ultra thick,red,-] (0.08, 0.08) -- (0.45, 0.45);
    },
    highlightedrelu3/.pic={
        \draw[thick,->] (-0.6, 0) -- (0.6, 0);
        \draw[thick,->] (0, -0.6) -- (0, 0.6);
        \draw[ultra thick,blue,-] (-0.45, 0) -- (0, 0);
        \draw[ultra thick,blue,-] (0, 0) -- (0.45, 0.45);
        \draw[ultra thick,red,-] (-0.6, 0) -- (-0.08, 0);
    },
    highlightedrelu4/.pic={
        \draw[thick,->] (-0.6, 0) -- (0.6, 0);
        \draw[thick,->] (0, -0.6) -- (0, 0.6);
        \draw[ultra thick,blue,-] (-0.45, 0) -- (0, 0);
        \draw[ultra thick,blue,-] (0, 0) -- (0.45, 0.45);
        \draw[ultra thick,red,-] (0.2, 0.2) -- (0.45, 0.45);
    },
}

\begin{document}

\maketitle

\begin{abstract}
    We consider the problem of exact computation of the maximum function over $d$ real inputs using ReLU neural networks. We prove a depth hierarchy, wherein width $\Omega\big(d^{1+\frac{1}{2^{k-2}-1}}\big)$ is necessary to represent the maximum for any depth $3\le k\le \log_2(\log_2(d))$. This is the first unconditional super-linear lower bound for this fundamental operator at depths $k\ge3$, and it holds even if the depth scales with $d$. Our proof technique is based on a combinatorial argument and associates the non-differentiable ridges of the maximum with cliques in a graph induced by the first hidden layer of the computing network, utilizing Turán's theorem from extremal graph theory to show that a sufficiently narrow network cannot capture the non-linearities of the maximum. This suggests that despite its simple nature, the maximum function possesses an inherent complexity that stems from the geometric structure of its non-differentiable hyperplanes, and provides a novel approach for proving lower bounds for deep neural networks.
\end{abstract}

    \section{Introduction}

    The role depth plays in the expressive power of neural networks and its benefits over width have been studied extensively in recent years \citep{eldan2016power,telgarsky2016benefits,daniely2017depth,yarotsky2017error,liang2017deep,safran2017depth,safran2019depth,venturi2021depth,hsu2021approximation,safran2022optimization,safran2024many,safran2024depth}. Despite our growing understanding of the settings in which depth is beneficial, establishing approximation lower bounds remains an extremely challenging problem, especially for deep neural networks. While machine learning applications are ultimately concerned with approximating target functions with respect to a data distribution in the $L_2$ sense, recent work has shifted toward simpler, and perhaps cleaner, problem settings where exact computation is considered instead \citep{hertrich2021towards,haase2023lower,bakaev2025depth,averkov2025expressiveness,grillo2025depth}. While it is well-known that sufficiently wide depth-2 networks can approximate any continuous function to arbitrary accuracy on a compact domain \citep{cybenko1989approximation,hornik1991approximation,leshno1993multilayer}, this is not the case for exact computation. For example, a piecewise-linear activation such as the Rectified Linear Unit (ReLU) cannot exactly represent a quadratic function. For this reason, a common setting for studying the limitations of deep neural networks focuses on the ability of ReLU networks to represent continuous piecewise-linear (CPWL) target functions. It is known that sufficiently deep ReLU networks can compute any CPWL function \citep{arora2016understanding,bakaev2025better}. This follows from the result that all CPWL functions can be represented as a linear combination of affine functions composed with the maximum function \citep{ovchinnikov2000max,wang2005generalization}.

    Following these results, many works have studied whether depth is indeed necessary for computing the maximum over $d$ inputs, which we denote by $\Max_d$. It is known that $\Max_3$ cannot be computed exactly by any depth-2 network \citep{mukherjee2017lower}. Conversely, since $\Max_2(x_1,x_2)=\max\{0,x_1\}-\max\{0,-x_1\}+\max\{0,x_2-x_1\}$ is easily computable using three ReLU neurons, one can compute $\Max_d$ using depth $\ceil{\log_2(d)} + 1$ by simulating a tournament structure --- taking pairwise maxima in each layer until the maximum is extracted \citep{arora2016understanding}. \citet{hertrich2021towards} conjectured that $\Max_5$ cannot be computed by depth-3 ReLU networks, but this was recently refuted by \citet{bakaev2025better}, who provided an explicit depth-3 construction for $\Max_5$ and a depth-($\lceil \log_3(d-2) \rceil + 1$) construction for $\Max_d$. Despite this progress, it remains an open problem whether $\Max_d$ can be computed by shallower networks. For instance, it is unknown if $\Max_6$ is computable at depth 3. To date, no CPWL function has been proven uncomputable by depth-3 ReLU networks without imposing specific restrictions on the weights or architecture \citep{hertrich2021towards,matoba2022theoretical,haase2023lower,safran2024many,bakaev2025depth,averkov2025expressiveness,grillo2025depth}.
    
    In this paper, rather than imposing weight or structural restrictions on the network to establish incomputability results, we investigate unconditional lower bound requirements for computing $\Max_d$. Specifically, we establish a depth-width hierarchy demonstrating that while the required width decays as the depth increases, a width super-linear in the input dimension $d$ is still required even when the depth $k$ scales with the input dimension. This result provides a quantitative characterization of the efficiency of depth, showing that while adding layers reduces the width requirement, a fundamental super-linear barrier persists beyond the constant-depth regime. More formally, our main result is the following:
    \begin{theorem}\label{thm:depthk}
        Suppose that $3\le k\le\log_2(\log_2(d))$. Let $\Ncal$ be a depth-$k$ ReLU network such that
        \[
            \Ncal(\bx)=\Max_d(\bx)
        \]
        for all $\bx\in[0,1]^d$. Then, $\Ncal$ has width at least
        \[
            0.1d^{1+\frac{1}{2^{k-2}-1}}.
        \]
    \end{theorem}
    Our result provides the first unconditional super-linear lower bound for the computation of $\Max_d$, and more generally for any $1$-Lipschitz continuous piecewise-linear function, using ReLU networks for all constant depths $k \ge 3$. We remark that our lower bound remains super-linear for all $k \le \log_2(\log_2(d)) - \omega(1)$ and naturally transitions to a linear bound as the depth approaches $\Ocal(\log_2(\log_2(d))$.

    Our proof technique proceeds by induction and is based on a combinatorial argument that links the piecewise-linear structure of the network with cliques in a graph induced by the arrangement of hyperplanes in the first hidden layer. A key component of our analysis is the following seminal result from extremal graph theory, which provides a fundamental bound on the density of graphs with excluded cliques.
    
    \begin{theorem}[Turán's Theorem \citep{turan1941extremal}]\label{thm:turan}
        Suppose that $G$ is a graph on $d$ vertices that does not contain a clique of size $r\ge3$. Then, $G$ has at most $\p{1-\frac{1}{r-1}}\frac{d^2}{2}$ edges.
    \end{theorem}
    
    A central challenge in deriving computation lower bounds for neural networks is that an impossibility result for depth $k-1$ rarely provides a clear path toward a generalization for depth $k$. Consequently, many known proof techniques are specialized to shallow architectures and offer little scalability. By leveraging Turán's theorem, we demonstrate that the non-linearities in the first hidden layer become redundant when restricted to a specific subset of input coordinates. This allows us to compute the maximum on this subset while effectively ``collapsing" the first hidden layer --- circumventing the inductive barrier and enabling a recursive application of the argument for deeper architectures. We remark that we did not attempt to optimize the constant factor $0.1$, though our method could potentially improve it to $0.5$ for sufficiently large $k$. We prioritized clarity and simplicity in the proof, given that the lower bound is asymptotically governed by the exponent. The reader is referred to Section~\ref{sec:proof_sketch} for a detailed proof sketch, and to Appendix~\ref{app:main_proof} for the complete formal proof.
    
    It is interesting to note that our exact computation lower bound nearly matches the approximation upper bound of \citet{safran2024many}, who prove that depth $2k-3$ and width $d^{1+\frac{1}{2^{k-2}-1}}$ for all $k\ge3$ suffice to obtain arbitrarily good $L_2$ approximation of $\Max_d$ with respect to continuous data distributions. Even though there exists a notable gap between the depth requirements for exact computation (depth $k$) and $L_2$ approximation (depth $2k-3$), our result suggests that the price of exactly matching the non-differentiable ridges of the $\Max_d$ function requires a width that grows at a rate nearly identical to the $L_2$ approximation case, despite the differences in depth and the notion of approximation. Since the approximation error in \citet{safran2024many} can be made arbitrarily small only by increasing the magnitude of the weights, this implies that the bottleneck for computing $\Max_d$ is not the error tolerance, but the inherent structural complexity of its $d$ hyperplanes.

    It is also natural to compare our result to analogs in threshold circuit complexity, where similar lower bounds are sought after for computing Boolean functions in the complexity class $\mathcal{P}$ using bounded-depth circuits that employ threshold activations as non-linearities. While double-exponentially decaying lower bounds of the form $d^{1+c\theta^{-k}}$ exist for the number of wires required for computing the parity function \citep{impagliazzo1997size}, the best-known exponent $\theta=1+\sqrt{2} \approx 2.414$ has remained the state-of-the-art for decades. In comparison, in this paper we obtain $c=4$ and $\theta=2$ for the required number of neurons (gates), which immediately implies a lower bound on the wires, and is stronger than what is known for threshold circuits. This discrepancy suggests potential new avenues for proving circuit lower bounds; however, such results do not follow immediately from neural network lower bounds, and we leave this connection as a tantalizing direction for future work. Notably, proving threshold circuit lower bounds for all $\theta>1$ for certain $\mathcal{NC}^1$-complete problems would imply that $\mathcal{TC}^0\neq\mathcal{NC}^1$, and would solve a major open problem in circuit complexity \citep{chen2019bootstrapping}. Despite this intriguing connection, it is not clear if our proof technique is capable of improving the exponent $\theta$ below $2$, and whether there exist efficient reductions to carry these parameters reliably to threshold circuits.

    In the following subsection, we turn to discuss additional related work. Thereafter, we introduce necessary notation and terminology, and in Section~\ref{sec:proof_sketch} we provide a detailed proof sketch of our result. Formal proofs are deferred to the appendix.

    \subsection{Additional related work}

    \paragraph{Lower bounds via a region-counting argument using deep neural networks}
    
    A simple and intuitive method for establishing lower bounds for CPWL functions involves bounding the maximal number of linear regions a network can realize along a one-dimensional line. This technique, pioneered by \citet{telgarsky2016benefits} to demonstrate a depth separation for highly oscillatory functions, typically requires the target function to have a Lipschitz constant that scales exponentially with depth --- a property often considered pathological in a learning context. While such methods can establish a depth hierarchy, this is achieved by gradually increasing the complexity of the target function via its Lipschitz parameter. In contrast, our result applies to the $\Max_d$ function, which remains $1$-Lipschitz regardless of the input dimension or network depth. This demonstrates that depth provides a fundamental representational advantage even for ``well-behaved" functions whose region complexity is small and thus simple region-counting arguments do not apply.

    \citet{yarotsky2017error, liang2017deep} and \citet{safran2017depth} employ similar region-counting arguments to demonstrate that smooth (and potentially Lipschitz) target functions can be well-approximated by deep architectures. These architectures generate a large number of linear regions and align them to resolve the curvature of the target function. While such results establish a depth hierarchy for approximation, there is significant evidence that these highly-oscillatory representations are difficult to learn efficiently \citep{hanin2019complexity, malach2021connection, vardi2021size}. In contrast, our work focuses on the natural $\Max_d$ function, which does not rely on depth-induced exponential oscillation for its expression. Consequently, our lower bound suggests that depth provides a fundamental representational advantage even for functions that do not suffer from the aforementioned learnability barriers.

    \paragraph{Lower bounds for $\Max_d$}
    \citet{safran2024many} show a width-$d$ lower bound that holds regardless of depth, but this is only a linear lower bound rather than super-linear as ours. Additionally, they establish an approximation lower bound of $1/\poly(d)$ (with respect to the uniform distribution over the unit hypercube) for depth-3 ReLU networks of width $\Ocal(d^2)$. While this result implies an exact computation lower bound, it relies on the additional assumption that the network weights are exponentially bounded. In contrast, our lower bound is unconditional and imposes no limitations on the magnitude of the weights, since we employ a novel construction that identifies an assignment of weights allowing for a recursive application of our main inductive argument. Lastly, while our proof builds upon the technique pioneered in \citet{safran2024many}, our proof is an adaptation and extension of it, since we generalize the technique to arbitrary depths by recursively applying Turán's theorem within an inductive framework, rather than merely applying Mantel's theorem as required in the induction base.
    
    \citet{mukherjee2017lower} and \citet{vsima2025power} establish that no depth-2 ReLU network can compute $\Max_3$, a result which is used in our induction base in our proof. However, the analysis in \citet{vsima2025power} assumes non-negative inputs, and it is not immediately clear if their approach extends to the compact domains we consider here. While our proof for depth-2 incomputability builds upon the techniques in \citet{mukherjee2017lower}, we identify a technical gap in their original argument regarding the characterization of non-differentiable points. Consequently, we provide a self-contained proof and a generalization that addresses this gap.\footnote{The authors in \citep{mukherjee2017lower} suggest that the set of non-differentiable points of a function computed by a depth-2 ReLU network is \emph{precisely} the union of the hyperplanes $\inner{\bw_i,\bx}+b_i=0$ defined by the $i^\text{th}$ hidden neuron. However, this oversight ignores certain edge cases where, for instance, the activation boundaries of two or more neurons coincide to mutually smooth their respective non-differentiable hyperplanes.} (See \lemref{lem:hyperplane_arrangement} and \propref{prop:max3_depth2} in the appendix.)

    \citet{hertrich2021towards} establish an auxiliary proposition similar to ours, demonstrating that a network computing a homogeneous function (see Subsection~\ref{subsec:notation} for a formal definition) can be ``homogenized" by eliminating all bias terms. While we utilize a similar homogenization argument, our approach differs in several key aspects. First, we do not assume the function is globally homogeneous, but only that it coincides with a homogeneous function on a compact domain. Second, our analysis requires a quantitative bound on the width of the resulting network to preserve the super-linear lower bound. Furthermore, we must prove that the homogenized network retains at least two non-zero coordinates in the weights of all first-layer neurons to satisfy the prerequisites for applying Turán's theorem. Consequently, our result can be viewed as a generalization of theirs, and our proof is significantly more involved as it necessitates tighter structural control that was not required in previous works.

    \paragraph{Reductions to circuit complexity lower bounds}

    Our results are fundamentally distinct from existing lower bounds inspired by threshold circuit complexity. While previous works have successfully leveraged reductions to circuit complexity or communication complexity to establish network size lower bounds \citep{mukherjee2017lower, vardi2021size}, those bounds typically yield linear constraints on the total number of neurons. In contrast, we establish a super-linear width requirement that persists even as the depth $k$ grows. Crucially, total size lower bounds often allow for a symmetric trade-off where depth and width are interchangeable resources (namely, doubling depth allows for halving width). Our results, however, characterize a non-symmetric trade-off. We demonstrate that for exact computation of $\Max_d$, depth is a strictly more efficient resource. This highlights that the architectural bottleneck in resolving the intersecting hyperplanes of $\Max_d$ is inherently tied to depth in a way that total size bounds cannot capture.
    
    \subsection{Preliminaries and notation}\label{subsec:notation}

    \paragraph{Notation and terminology}
    We let $[n]$ be shorthand for the set $\{1,\ldots,n\}$. We denote vectors using bold-faced letters (e.g.\ $\bx$). Given a vector $\bx=(x_1,\ldots,x_d)$, we let $\norm{\bx}$ denote its Euclidean norm. Throughout, we use the notation $\Max_d(\bx)\coloneqq\max\{x_1,\ldots,x_d\}$ for the maximum function, and $\relu{x}=\max\{0,x\}$ for the ReLU activation function. A function or network $f$ is (positively) homogeneous if for all $c>0$ and all $\bx\in\reals^d$, $f(c\bx)=cf(\bx)$. A function $f:D\to\reals$ defined in some domain $D\subseteq\reals^d$ is \emph{continuous piecewise-linear (CPWL)} if there exists a finite partition $D=\cup_{i}D_i$ such that $f$ is linear on $D_i$ for all $i$, where each $D_i$ is a closed set.

    \paragraph{Neural networks}

    We consider fully connected, feed-forward neural networks, computing functions from $\reals^d$ to $\reals$. We focus on neural networks which employ the ReLU function as a non-linear activation. A ReLU neural network consists of layers of neurons, where in every layer except for the output neuron, an affine function of the inputs is computed, followed by a computation of the non-linear activation function $\relu{\cdot}$. The single output neuron simply computes an affine transformation of its inputs. Each layer with a non-linear activation is called a \emph{hidden layer}, and the \emph{depth} of a network is defined as the number of hidden layers plus one, and is generally denoted by $k$. The \emph{width} of a network is defined as the number of neurons in the largest hidden layer, and the \emph{size} of the network is the total number of neurons across all layers.
        
    \section{Techniques and proof sketch}\label{sec:proof_sketch}

    In this section, we provide a detailed proof sketch of Theorem~\ref{thm:depthk}. Our proof is centered on the geometric intuition that a ReLU network of depth $k$ computing $\Max_d$ must account for $\binom{d}{2}$ distinct non-linearities: one for each pair of coordinates $i,j$ that takes effect when $x_i$ and $x_j$ are the maximal values in $\bx \in \reals^d$, and one ``overtakes" the other. If these non-linearities are absent in the first hidden layer, the burden of computing them must fall on the subsequent, deeper layers.

    Specifically, if the network is sufficiently narrow, we use Turán's theorem to identify a clique of size $r$ among the input coordinates for which these pairwise non-linearities are missing in the first layer. A key technical proposition allows us to transition from a network computing $\Max_d$ on a compact domain to a network of the same depth (and at most twice the width) that computes $\Max_d$ on all of $\reals^d$. By identifying an assignment of extremely negative values for the $d-r$ coordinates outside the clique, we can ensure that the neurons in the first hidden layer never change their activation state regardless of the values of the remaining $r$ coordinates, that are all bounded in $[0,1]$. This reduces the first hidden layer to a simple linear transformation of its input, allowing us to collapse it and construct a depth-$(k-1)$ network that computes $\Max_r$ on $[0,1]^r$. The theorem then follows by applying our induction hypothesis to this reduced architecture.

    Below, we further detail each significant step in the proof, starting from the induction base.

    \subsection{Step 1: The induction base --- a depth-3, width-$0.1d^2$ lower bound}
        As the base case, we consider depth-3 neural networks, and we show that a width of approximately $\frac{1}{8}d^2$ is necessary for computing $\Max_d$. More formally, we prove the following lower bound.
        \begin{theorem}\label{thm:depth3}
            Let $d\ge1$ and suppose that $\Ncal$ is a depth-3 ReLU network such that
            \[
                \Ncal(\bx)=\Max_d(\bx)
            \]
            for all $\bx\in[0,1]^d$. Then, $\Ncal$ has width at least
            \[
                \floor{\p{\frac18-\frac{1}{4d}-\frac{1}{2d^2}}d^2}.
            \]
        \end{theorem}
        The proof of the above theorem, which appears in \appref{app:depth3_proof}, is based on a similar technique which was developed in \citet{safran2024many}, and uses Mantel's theorem (a particular case of Turán's theorem for cliques of size 3) to identify a triplet of coordinates on which the network computes $\Max_3$ while maintaining a fixed activation pattern in the first layer.

        As $d \to \infty$, this lower bound approaches $\frac{1}{8}d^2$, which represents the optimal constant achievable for depth 3 using our current framework.
        
    \subsection{Step 2: From computation on the unit hypercube to $\reals^d$}\label{subsec:step2}

    Following the induction base, we now detail the derivation of the inductive step. This stage is critical for Step 4 (\subsecref{subsec:step4}), where we perform a substitution of extremely negative values while ensuring we remain within the valid domain of computation for $\Max_d$. To this end, we prove \propref{prop:compact_to_unbounded}, which establishes that the efficient computation of $\Max_d$ on a compact domain implies its efficient computation on all of $\reals^d$.
    
    While this part of the proof is inspired by the ``homogenization" argument in \citet{hertrich2021towards}, our approach differs in two key aspects: unlike their proof which crucially relies on the target function being homogeneous, we only assume that this is the case on a compact domain and thus their proof technique is not directly applicable; and we guarantee the existence of a network where all first-layer weights have at least two non-zero coordinates, which is a prerequisite for the graph-theoretic analysis in Step~3 (\subsecref{subsec:step3}).
    
    We begin by shifting the network with a random bias, allowing it to compute the maximum in a neighborhood of the origin. The proof of the proposition then proceeds by recursively ``pushing" non-zero biases from the first hidden layer forward through the network. This process preserves the function's behavior in a sufficiently small neighborhood of the origin through the following logic:
    \begin{itemize}
        \item
        Negative biases: If a neuron has a negative bias, it becomes inactive in a small enough neighborhood of the origin and can be removed without affecting the output (see Figure~\ref{fig:homogenize_negative_bias} for an illustration).
        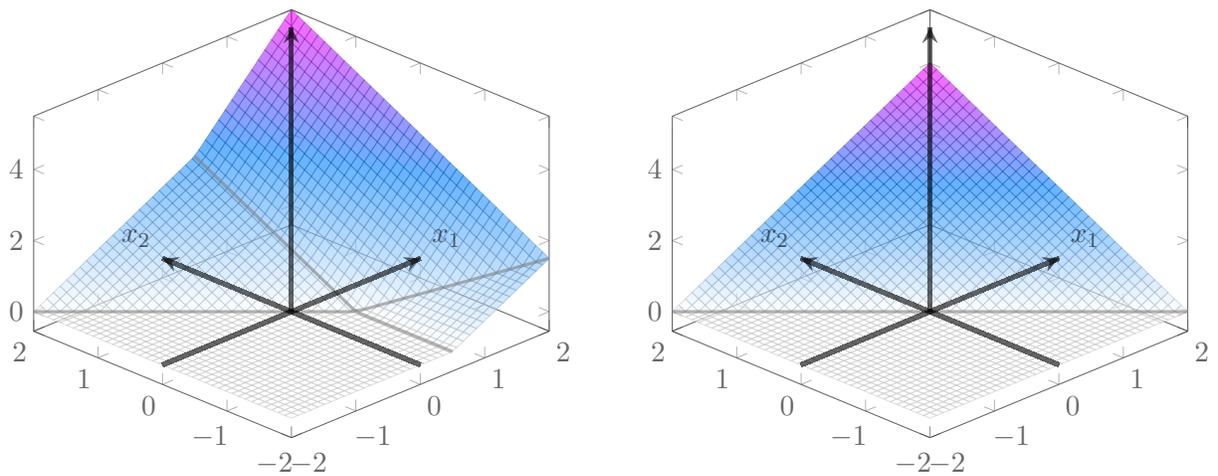
\begin{figure}
            \centering
            \begin{tikzpicture}[scale=1]
                \begin{axis}[
                    view={315}{35},
                    domain=-2:2,
                    domain y=-2:2,
                    samples=40,
                    samples y=40,
                    colormap/cool,
                    zmax=5.5,
                    opacity=0.6,  
                    ]
                    \addplot3[surf,grid=none]{max(0, x+y) + max(0, x - 0.5)};
        
                    \addplot3[
                        gray, 
                        very thick, 
                        samples y=0, 
                        domain=-2:2, 
                        on layer=axis foreground
                    ] 
                    (x, -x, {max(0, x - 0.5)});
            
                    \addplot3[
                        gray, 
                        very thick, 
                        samples y=0, 
                        domain=-2:2, 
                        on layer=axis foreground
                    ] 
                    (0.5, x, {max(0, 0.5 + x)});
        
                    \addplot3[black, ultra thick, -stealth, samples=2, domain=-2:2, on layer=axis foreground] 
                        (x, 0, 0) node[above right] {$x_1$};
                    
                    \addplot3[black, ultra thick, -stealth, samples=2, domain=-2:2, on layer=axis foreground] 
                        (0, x, 0) node[above left] {$x_2$};
                    
                    \addplot3[black, ultra thick, -stealth, samples=2, domain=0:8, on layer=axis foreground] 
                        (0, 0, x) node[above] {};
                \end{axis}
            \end{tikzpicture}\quad\quad
            \begin{tikzpicture}[scale=1]
                \begin{axis}[
                    view={315}{35},
                    domain=-2:2,
                    domain y=-2:2,
                    samples=40,
                    samples y=40,
                    colormap/cool,
                    zmax=5.5,
                    opacity=0.6,  
                    ]
                    \addplot3[surf,grid=none]{max(0, x+y)};
            
                    \addplot3[
                        gray, 
                        very thick, 
                        samples y=0, 
                        domain=-2:2, 
                        on layer=axis foreground
                    ] 
                    (x, -x, 0);

                    \addplot3[black, ultra thick, -stealth, samples=2, domain=-2:2, on layer=axis foreground] 
                        (x, 0, 0) node[above right] {$x_1$};
                    
                    \addplot3[black, ultra thick, -stealth, samples=2, domain=-2:2, on layer=axis foreground] 
                        (0, x, 0) node[above left] {$x_2$};
                    
                    \addplot3[black, ultra thick, -stealth, samples=2, domain=0:8, on layer=axis foreground] 
                        (0, 0, x) node[above] {};
                \end{axis}
            \end{tikzpicture}
            \caption{The function $\relu{x_1+x_2}+\relu{x_1-0.5}$ (left) and the resulting function $\relu{x_1+x_2}$ (right) when the neuron with negative bias is removed. Since the neuron is inactive in a neighborhood of the origin, removing it maintains the computation near the origin.}
            \label{fig:homogenize_negative_bias}
        \end{figure}
        \item 
        Positive biases: If a neuron has a positive bias, the affine transformation it computes can be simulated near the origin by substituting it with two unbiased (homogeneous) neurons and adjusting the biases in the subsequent layer (see Figure~\ref{fig:homogenize_positive_bias} for an illustration).
        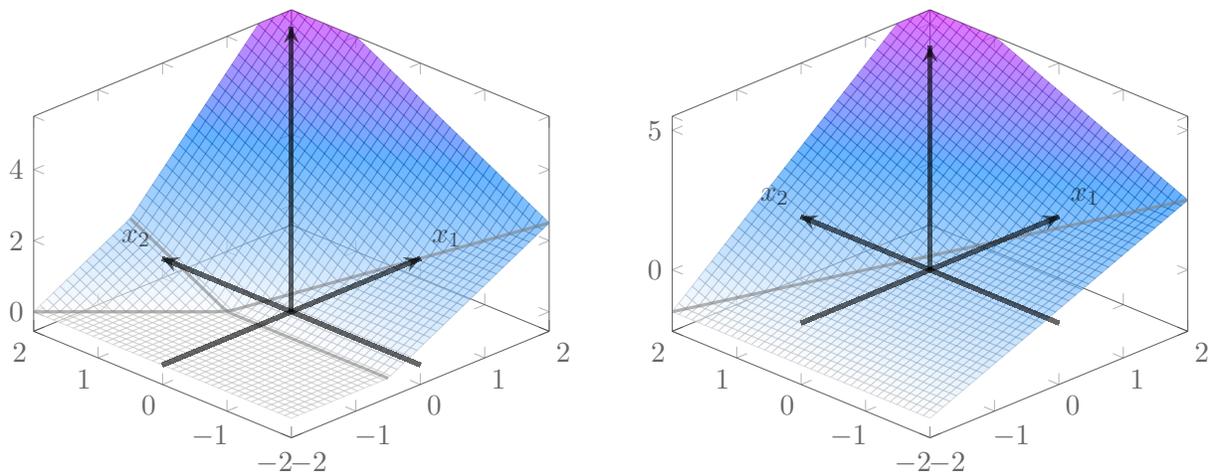
\begin{figure}
            \centering
            \begin{tikzpicture}[scale=1]
                \begin{axis}[
                    view={315}{35},
                    domain=-2:2,
                    domain y=-2:2,
                    samples=40,
                    samples y=40,
                    colormap/cool,
                    zmax=5.5,
                    opacity=0.6,  
                    ]
                    \addplot3[surf,grid=none]{max(0, x+y) + max(0, x + 0.5)};
        
                    \addplot3[
                        gray, 
                        very thick, 
                        samples y=0, 
                        domain=-2:2, 
                        on layer=axis foreground
                    ] 
                    (x, -x, {max(0, x + 0.5)});
            
                    \addplot3[
                        gray, 
                        very thick, 
                        samples y=0, 
                        domain=-2:2, 
                        on layer=axis foreground
                    ] 
                    (-0.5, x, {max(0, -0.5 + x)});
        
                    \addplot3[black, ultra thick, -stealth, samples=2, domain=-2:2, on layer=axis foreground] 
                        (x, 0, 0) node[above right] {$x_1$};
                    
                    \addplot3[black, ultra thick, -stealth, samples=2, domain=-2:2, on layer=axis foreground] 
                        (0, x, 0) node[above left] {$x_2$};
                    
                    \addplot3[black, ultra thick, -stealth, samples=2, domain=0:8, on layer=axis foreground] 
                        (0, 0, x) node[above] {};
                \end{axis}
            \end{tikzpicture}\quad\quad
            \begin{tikzpicture}[scale=1]
                \begin{axis}[
                    view={315}{35},
                    domain=-2:2,
                    domain y=-2:2,
                    samples=40,
                    samples y=40,
                    colormap/cool,
                    zmax=5.5,
                    opacity=0.6,  
                    ]
                    \addplot3[surf,grid=none]{max(0, x+y) + max(0, x) - max(0, -x) + 0.5};
        
                    \addplot3[
                        gray, 
                        very thick, 
                        samples y=0, 
                        domain=-2:2, 
                        on layer=axis foreground
                    ] 
                    (x, -x, {x + 0.5});
        
                    \addplot3[black, ultra thick, -stealth, samples=2, domain=-2:2, on layer=axis foreground] 
                        (x, 0, 0) node[above right] {$x_1$};
                    
                    \addplot3[black, ultra thick, -stealth, samples=2, domain=-2:2, on layer=axis foreground] 
                        (0, x, 0) node[above left] {$x_2$};
                    
                    \addplot3[black, ultra thick, -stealth, samples=2, domain=0:8, on layer=axis foreground] 
                        (0, 0, x) node[above] {};
                \end{axis}
            \end{tikzpicture}
            \caption{The function $\relu{x_1+x_2}+\relu{x_1+0.5}$ (left) and the resulting function $\relu{x_1+x_2}+\relu{x_1}-\relu{-x_1}+0.5$ (right) when the neuron with positive bias is removed and replaced with two homogeneous neurons, and the bias of the output neuron is modified. The function maintains its behavior near the origin while pushing biases one layer forward.}
            \label{fig:homogenize_positive_bias}
        \end{figure}
    \end{itemize}
    Since each original neuron is replaced by at most two homogeneous neurons, the width of the resulting network is at most twice that of the original. Because the final network is homogeneous and coincides with $\Max_d$ near the origin, a homogeneity argument implies they must coincide on all of $\reals^d$. Finally, we demonstrate that with probability one, all neurons in the constructed network possess at least two non-zero weights, facilitating the next step of the proof.
    
    \subsection{Step 3: Constructing a first-layer weight graph with a large clique}\label{subsec:step3}
    
    In this step, we establish a formal connection between the weights of the first hidden layer and extremal graph theory. Based on the technique in \citet{safran2024many}, we construct a graph $G_{\mathcal{N}}$ derived from the weight vectors of the width-$n$ network $\mathcal{N}$ as follows:
    \begin{itemize}
        \item
        We define a graph on $d$ vertices, where each vertex represents an input coordinate.
        \item
        For each neuron $i$ in the first hidden layer, let $\bw_i \in \reals^d$ be its weight vector. We remove an edge $(j, k)$ from the initial complete graph $K_d$ based on the indices of the non-zero entries in $\bw_i$. Specifically, if $\bw_i$ contains at least two non-zero elements, we remove the edge which corresponds to the two smallest indices.
        \item
        Since each of the $n$ neurons can remove at most one edge, the resulting graph $G_{\mathcal{N}}$ contains at least $\binom{d}{2} - n$ edges.
        \item
        If $n$ is sufficiently small, then by Turán's theorem (Theorem~\ref{thm:turan}), $G_{\mathcal{N}}$ must contain a clique of size $r$.
    \end{itemize}
    This clique $I \subseteq \{1, \dots, d\}$ identifies a subset of coordinates with a crucial property: for any pair of coordinates $j, k \in I$, there is no neuron in the first hidden layer whose activation is determined solely by $x_j$ and $x_k$ (see Figure~\ref{fig:induced_graph} for an illustration). This implies that the first layer is ``blind" to the pairwise interactions of these coordinates, which forces the burden of computing the transitions of $\Max_r$ onto the subsequent $k-1$ layers. This structural gap is what enables the layer-reduction step in the final stages of the proof.
    \begin{figure} [h]
        \centering
        \begin{tikzpicture}[scale=1.2]
            \definecolor{colorw1}{RGB}{200, 0, 0}   
            \definecolor{colorw2}{RGB}{0, 0, 200}   
            \definecolor{colorw3}{RGB}{0, 150, 0}   
            \definecolor{colorw4}{RGB}{200, 120, 0} 
    
            \begin{scope}[xshift=-4cm]
                \foreach \i in {1,...,5} {
                    \node[circle, fill=black, inner sep=1.5pt] (u\i) at ({90 - (\i-1)*72}:1.3) {};
                    \node at ({90 - (\i-1)*72}:1.6) {\small $v_\i$};
                }
                \begin{scope}[very thick]
                    \draw[color=colorw1] (u3) -- (u4);
                    \draw[color=colorw2] (u2) -- (u4);
                    \draw[color=colorw3] (u1) -- (u5);
                    \draw[color=colorw4] (u4) -- (u5);
                    \draw (u1) -- (u2); \draw (u1) -- (u3); \draw (u1) -- (u4);
                    \draw (u2) -- (u3); \draw (u2) -- (u5); \draw (u3) -- (u5);
                \end{scope}
            \end{scope}
    
            \begin{scope}[xshift=0cm]
                \node[anchor=west, color=colorw1] at (-1.5, 0.9)  {\small $\mathbf{w}_1 \coloneqq (0, 0, 1, -1, 0)$};
                \node[anchor=west, color=colorw2] at (-1.5, 0.3)  {\small $\mathbf{w}_2 \coloneqq (0, -1, 0, 1, 0)$};
                \node[anchor=west, color=colorw3] at (-1.5, -0.3) {\small $\mathbf{w}_3 \coloneqq (1, 0, 0, 0, -1)$};
                \node[anchor=west, color=colorw4] at (-1.5, -0.9) {\small $\mathbf{w}_4 \coloneqq (0, 0, 0, 1, -1)$};
                
            \end{scope}
    
            \begin{scope}[xshift=4cm]
                \foreach \i in {1,...,5} {
                    \node[circle, fill=black, inner sep=1.5pt] (v\i) at ({90 - (\i-1)*72}:1.3) {};
                    \node at ({90 - (\i-1)*72}:1.6) {\small $v_\i$};
                }
                \begin{scope}[very thick]
                    \draw (v1) -- (v2);
                    \draw (v1) -- (v3);
                    \draw (v1) -- (v4);
                    \draw (v2) -- (v3);
                    \draw (v2) -- (v5);
                    \draw (v3) -- (v5);
                \end{scope}
            \end{scope}
    
        \end{tikzpicture}
        \caption{The construction of the graph $G_{\mathcal{N}}$ that is induced by the first-hidden-layer weights of $\Ncal$. \textbf{Left:} The complete graph $K_5$ where edges are colored according to the first-layer neuron $\mathbf{w}_i$ that removes them. \textbf{Middle:} The weight vectors of the first hidden layer whose non-zero coordinates dictate the edge removal process. \textbf{Right:} The final graph $G_{\mathcal{N}}$ after the removal of all colored edges. If the width $n$ is small relative to $d$, Turán's theorem ensures that the graph still contains a large clique (e.g., the triangle formed by the first three coordinates).}
        \label{fig:induced_graph}
    \end{figure}
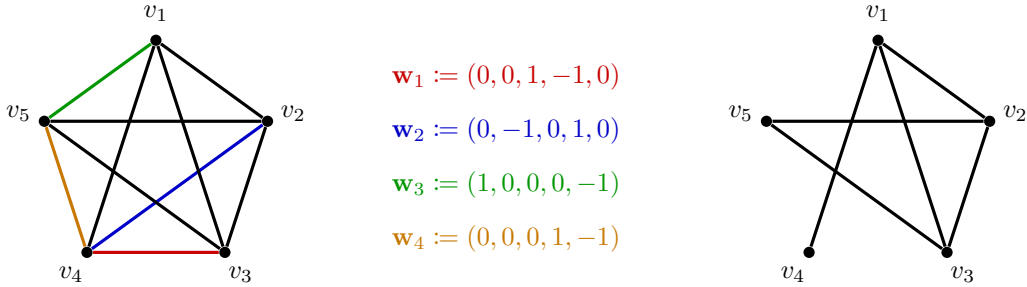

    \subsection{Step 4: Dimensionality reduction via coordinate assignment}\label{subsec:step4}Utilizing the structural property of the weight-graph established in Step 3 above, we now construct an assignment of inputs that forces the network to compute a lower-dimensional version of the maximum function. Recall that in Step 2 (\subsecref{subsec:step2}), we extended the computation of $\Max_d$ to all of $\mathbb{R}^d$ while ensuring that every neuron in the first hidden layer has a weight vector with at least two non-zero coordinates.
    
    Assume, without loss of generality, that the clique $I$ identified in the previous step corresponds to the first $r$ coordinates. We prove that there exists a substitution of sufficiently negative values for the remaining coordinates, $x_{r+1}, \dots, x_d$, such that the resulting network computes $\Max_r$ on the unit hypercube $[0,1]^r$. The core of this argument lies in assigning exponentially increasing negative values to these coordinates. By doing so, we ensure that for each neuron, the non-zero coordinate with the largest index (among the indices $r+1, \dots, d$) dominates the pre-activation, effectively dictating the activation pattern of the neuron for all $\bx' \in [0,1]^r$.
    
    Crucially, our construction of $G_{\mathcal{N}}$ in Step 3 ensures that every neuron in the first layer has at least one non-zero weight associated with a coordinate outside the clique $I = \{1, \dots, r\}$. Consequently, for any input in $[0,1]^r$, these neurons remain in a fixed activation state (either ``always on" or ``always off"). Since fixing these inputs preserves the network architecture while reducing the first hidden layer to a simple linear transformation, we have effectively constructed a neural network of the same depth that computes $\Max_r$ on $[0,1]^r$ without utilizing any non-linearities in the first hidden layer. (See Figure~\ref{fig:assignment} for an illustration.)

    \begin{figure}[h]
        \centering
        \begin{tikzpicture}[scale=0.75,
            input square/.style={
                rectangle, 
                draw=black, 
                fill=black!5, 
                very thick, 
                minimum size=4mm, 
                inner sep=0pt
            },
            conn/.style={
                ->, 
                >=stealth, 
                gray!50, 
                semithick,
                shorten >=1pt, 
                shorten <=1pt
            },
            label text/.style={font=\small}
        ]
        
            \def\dx{3.5} 
            \def\dy{2.2} 
        
    
            \node[input square, label={[label text]left:$x_1\in[0,1]$}] (I-1) at (-0.5, 4) {};
            \node[input square, label={[label text]left:$x_2\in[0,1]$}] (I-2) at (-0.5, 2) {};
            \node[input square, label={[label text]left:$x_3\in[0,1]$}] (I-3) at (-0.5, 0) {};
            \node[input square, label={[label text]left:$x_4=-\infty$}] (I-4) at (-0.5, -2) {};
            \node[input square, label={[label text]left:$x_5=-\infty$}] (I-5) at (-0.5, -4) {};

            \node[neuron] (H1-1) at (\dx, 1.5*\dy) {};
            \pic[scale=1.05] at (H1-1) {highlightedrelu1};
            \node[neuron] (H1-2) at (\dx, 0.5*\dy) {};
            \pic[scale=1.05] at (H1-2) {highlightedrelu2};
            \node[neuron] (H1-3) at (\dx, -0.5*\dy) {};
            \pic[scale=1.05] at (H1-3) {highlightedrelu3};
            \node[neuron] (H1-4) at (\dx, -1.5*\dy) {};
            \pic[scale=1.05] at (H1-4) {highlightedrelu4};
            
            \foreach \layer in {2,3} {
                \foreach \i in {1,...,4} {
                    \pgfmathsetmacro{\ypos}{(2.5 - \i) * \dy}
                    \node[neuron] (H\layer-\i) at (\layer*\dx, \ypos) {};
                    \pic[scale=1.05] at (H\layer-\i) {reluactivation};
                }
            }
        
            \node[neuron] (O) at (4*\dx, 0) {};
            \pic[scale=1.05] at (O) {linearactivation};
        
            \foreach \i in {1,...,5} \foreach \j in {1,...,4} \draw[conn] (I-\i) -- (H1-\j);
            \foreach \i in {1,...,4} \foreach \j in {1,...,4} \draw[conn] (H1-\i) -- (H2-\j);
            \foreach \i in {1,...,4} \foreach \j in {1,...,4} \draw[conn] (H2-\i) -- (H3-\j);
            \foreach \i in {1,...,4} \draw[conn] (H3-\i) -- (O);
        
        \end{tikzpicture}
        \caption{An illustration of the effects of the negative assignment of values explained in Step~4 (Subsection~\ref{subsec:step4}). By constructing a clique from the first three coordinates, we ensure that every neuron in the first hidden layer possesses a non-zero weight for some coordinate $j \ge 4$. Since the clique coordinates are bounded in $[0,1]$, assigning sufficiently large negative values to $x_4$ and $x_5$ ensures that each neuron's pre-activation is dominated by the term corresponding to its largest non-zero weight index. Consequently, these operating ranges (highlighted in red) remain strictly within either the positive or negative rays of the ReLU, rendering the first hidden layer's non-linearities redundant. Crucially, as Step~2 (Subsection~\ref{subsec:step2}) guarantees that the network computes $\text{Max}_d$ globally, this negative assignment forces the maximum to be attained by one of the first three coordinates. This effectively reduces the network's computation to $\text{Max}_3$ on the clique's domain, facilitating the lower bound proof.}
        \label{fig:assignment}
    \end{figure}
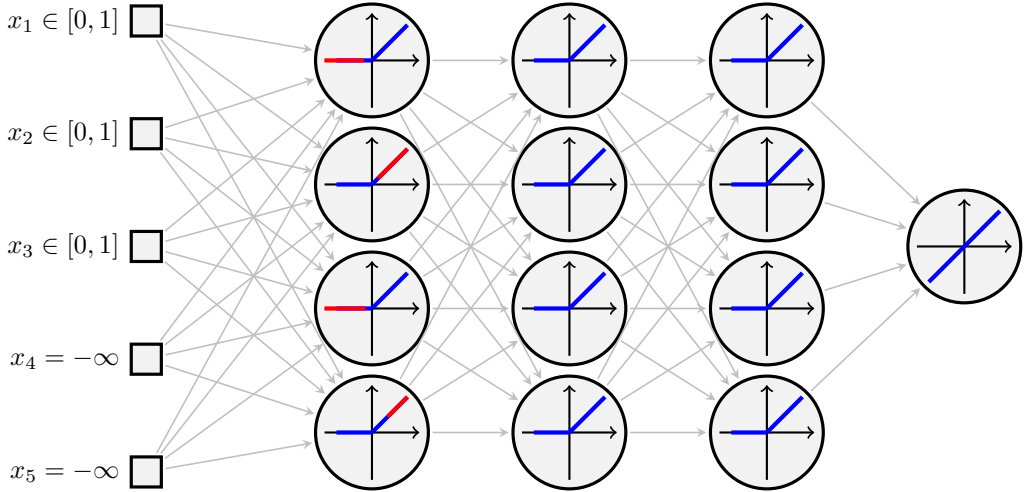

    \subsection{Step 5: Layer collapse and inductive contradiction}
    
    In this final step, we consolidate the previous results to complete the proof. From the construction in Step 4 above, we have obtained a network that computes $\Max_r$ on $[0,1]^r$, but where every neuron in the first hidden layer maintains a fixed activation pattern across the entire domain. This property allows us to ``collapse" the first hidden layer without altering the function computed by the network.
    
    Because the first layer is restricted to its linear region, it effectively computes a simple linear transformation of its input. This transformation can be absorbed into the weights of the second hidden layer, resulting in a new network that computes $\Max_r$ on $[0,1]^r$ with its depth reduced by one to $k-1$, and its width preserved. Finally, we apply our induction hypothesis to this reduced architecture. If the original network were too narrow, the resulting clique size $r$ (determined by Turán’s theorem) would be large enough to violate the width requirements for depth-$(k-1)$ networks, thus yielding a contradiction and concluding the proof of the theorem.
    
    \subsection*{Acknowledgments}
    This research is supported by Israel Science Foundation Grant 1753/25. I thank Dean Doron for helpful discussions regarding connections to circuit complexity.

    \bibliographystyle{abbrvnat}
    \bibliography{citations}

    \appendix

    \section{Proofs}
    
    \subsection{Auxiliary Lemmas}
    
    We first define a hyperplane arrangement as the finite union of hyperplanes in $\reals^d$. With this definition, we can define our first auxiliary lemma, used to prove the limitations of depth-2 ReLU networks for approximating piecewise linear functions. 
    
    \begin{lemma}\label{lem:hyperplane_arrangement}
        The set of non-differentiable points of a depth-2, width-$n$ ReLU network is a hyperplane arrangement with at most $n$ hyperplanes.
    \end{lemma}
    
    The proof technique of the above lemma is based on the approach of \citet[Proposition~2.2]{mukherjee2017lower}, and completes a gap in their original proof.
    
    \begin{proof}
        Let 
        \[
            \Ncal(\bx)=\sum_{i=1}^nv_j\relu{\inner{\bw_i,\bx}+b_i}+b_0
        \]
        be a depth-2, width-$n$ neural network. Given the weights $\bzero\neq\bw\in\reals^d$ and bias $b\in\reals$ of a hidden neuron, we define \[
            H_{\bw,b}\coloneqq \{\bx:\inner{\bw,\bx}+b=0\}
        \]
        as shorthand for the hyperplane induced by $\bw$ and $b$. To prove the lemma, we will manipulate the network, reducing the number of neurons in it while keeping its set of non-differentiable points unchanged. To this end, we have the following case analysis:
        \begin{itemize}
            \item
            If $v_j$ for some $j\in[n]$, then the output neuron ignores the $j^{\text{th}}$ neuron's output, and we can simply remove it from the network while keeping the function computed by $\Ncal$ unchanged.
            \item 
            Similarly, if $\bw_j=\bzero$ for some $j$, then we have
            \[
                v_j\relu{\inner{\bw_j,\bx}+b_j} = v_j\relu{b_j},
            \]
            and therefore the $j^{\text{th}}$ neuron only adds a constant to the computation of $\Ncal$ without changing its set of non-differentiable points, and we can thus remove it.\footnote{Note that while this might change the function computed by $\Ncal$, it does not change the set of non-differentiable points, and therefore removing it suffices for our purposes. However, if one wishes to remove all neurons such that $\bw_j=\bzero$ while keeping the function computed by $\Ncal$ unchanged, then this can be done by removing the $j^{\text{th}}$ neuron and modifying the output neuron's bias term using the transformation $b_0\mapsto b_0+v_j\relu{b_j}$, which would simulate the same computation.}
            \item 
            Suppose that $H_{\bw_{j},b_{j}}=H_{\bw_{k},b_{k}}$. Namely, we have that neuron $j$ and neuron $k$ for $j\neq k$ induce the same hyperplane. Since the solution space of the linear system 
            \[
                \p{
                \begin{matrix}
                    \bw_{j}^{\top}\\ \bw_{k}^{\top}
                \end{matrix}}
                \cdot\bx=-\p{
                \begin{matrix}
                    b_{j}\\ b_{k}
                \end{matrix}}
            \]
            is a rank $d-1$ subspace (the hyperplane induced by the neurons), $\bw_{j},\bw_{k}$ must be linearly dependent. Writing $\bw_{k}=\alpha\bw_{j}$ for some $\alpha\neq0$ and plugging this in the linear system, we obtain $b_{k}=\alpha b_{j}$. We now analyze two complementary cases:
            \begin{itemize}
                \item 
                Suppose that $H_{\bw_{j},b_{j}}=H_{\bw_{k},b_{k}}$ and that $\bw_{k}=\alpha\bw_{j}$ and $b_{k}=\alpha b_{j}$ for some $\alpha>0$. Then for all $v_{j},v_{k}\in\reals$ we have
                \begin{align*}
                    v_{j}\relu{\inner{\bw_{j},\bx} + b_{j}} + v_{k}\relu{\inner{\bw_{k},\bx} + b_{k}} &= v_{j}\relu{\inner{\bw_{j},\bx} + b_{j}} + v_{k}\relu{\inner{\alpha\bw_{j},\bx} + \alpha b_{j}}\\
                    &= v_{j}\relu{\inner{\bw_{j},\bx} + b_{j}} + \alpha v_{k}\relu{\inner{\bw_{j},\bx} + b_{j}}\\
                    &= (v_j + \alpha v_{k})\relu{\inner{\bw_{j},\bx} + b_{j}},
                \end{align*}
                where in the second equality we used the facts that the ReLU activation is positively homogeneous and $\alpha>0$. This shows that in this case, we can replace every pair of neurons with overlapping induced hyperplanes with a single neuron whose weights and bias equal $\bw_j$ and $b_j$, respectively, and modify the incoming weight of the output neuron to $v_j+\alpha v_k$. This would allow $\Ncal$ to compute the same function, thus keeping its set of non-differentiable points unchanged.
                \item 
                Suppose that $H_{\bw_{j},b_{j}}=H_{\bw_{k},b_{k}}$ and that $\bw_{k}=\alpha\bw_{j}$ and $b_{k}=\alpha b_{j}$ for some $\alpha<0$. Then in this case, we can assume that there are at most two neuron with overlapping induced hyperplanes (since otherwise we can cancel them as detailed in the previous item). Suppose that $v_j=\alpha v_k$. Then we have
                \begin{align}
                    v_{j}\relu{\inner{\bw_{j},\bx} + b_{j}} + v_{k}\relu{\inner{\bw_{k},\bx} + b_{k}} &= v_{j}\relu{\inner{\bw_{j},\bx} + b_{j}} + v_{k}\relu{\inner{\alpha\bw_{j},\bx} + \alpha b_{j}}\nonumber\\
                    &= v_{j}\relu{\inner{\bw_{j},\bx} + b_{j}} - \alpha v_{k}\relu{\inner{-\bw_{j},\bx} - b_{j}}\nonumber\\
                    &= v_{j}\p{\relu{\inner{\bw_{j},\bx} + b_{j}} - \relu{-\p{\inner{\bw_{j},\bx} + b_{j}}}}\nonumber\\
                    &= v_{j}\p{\inner{\bw_{j},\bx} + b_{j}},\label{eq:slope_cancels}
                \end{align}
                where in the second equality we used the positive homogeneity of the ReLU with $\alpha<0$, and in the last equality we used the fact that $\relu{z}-\relu{-z}=z$ for all $z\in\reals$. We can therefore remove neurons $j$ and $k$ without changing the non-differentiable set of $\Ncal$.
            \end{itemize}
            \item
            In all possible remaining cases, we have that a hyperplane is induced by a single neuron, or that it is induced by at most two neurons but with a negative $\alpha$. We will now construct the set of non-differentiable points of $\Ncal$ based on this derivation.
            \begin{itemize}
                \item
                In the first case, a set of non-differentiable points is formed along this hyperplane. Taking the union over all such hyperplanes, it is straightforward to verify that this union also forms a set of non-differential points. This is trivial in the case where a point does not intersect any other hyperplane, and when it does, it is easy to verify that it is still a non-differentiable point by considering the derivative along any direction that is not contained in one of the hyperplanes (such a direction must exist since there is a finite intersection of hyperplanes).
                \item 
                In the second case where a hyperplane is induced by two neurons, since their $\alpha$ is negative and since the slope in the direction perpendicular to the hyperplane changes (otherwise this pair of neurons would have merged into a single neuron according to Equation~(\ref{eq:slope_cancels})), we have that by the same reasoning as in the previous item, that the union of these hyperplanes with the hyperplanes in the previous item forms the set of non-differentiable points of $\Ncal$.
            \end{itemize}
        \end{itemize}
        We have constructed a hyperplane arrangement consisting of at most $n$ hyperplanes (since we can only remove but never add neurons and their induced hyperplanes) which forms the precise set of non-differentiable points of $\Ncal$, concluding the proof of the lemma.
    \end{proof}
    The following proposition shows that no depth-2 network can compute the function $\Max_3(\bx)$ on $[0,1]^3$.
    
    \begin{proposition}\label{prop:max3_depth2}
        There exists no depth-2 ReLU network which satisfies
        \[
            \Ncal(\bx)=\Max_3(\bx)
        \]
        for all $\bx\in[0,1]^3$.
    \end{proposition}
    
    \begin{proof}
        Suppose by contradiction that there exists a depth-2, width-$n$ ReLU network $\Ncal$ that computes $\Max_3$ for all $\bx\in[0,1]^3$. Consider the network $\tilde{\Ncal}$ obtained from $\Ncal$ by transforming all the biases of the hidden layer neurons using the transformation $b_i\mapsto b_i + 0.5\sum_{j=1}^3 w_{i,j}$, where $w_{i,j}$ is the $j^{\text{th}}$ coordinate of the weight of the $i^{\text{th}}$ neuron in the hidden layer. For all $\bx\in[0,1]^3$, let $\bx_{0.5}\coloneqq\bx-(0.5,0.5,0.5)$. Then we have that $\bx_{0.5}\in[-0.5,0.5]^3$ and that any neuron in the first hidden layer now computes for any $\bx_{0.5}$
        \begin{align*}
            \relu{\inner{\bw_i,\bx_{0.5}}+b_i-0.5\sum_{j=1}^d w_{i,j}} &= \relu{\inner{\bw_i,\bx-(0.5,0.5,0.5)}+b_i+0.5\sum_{j=1}^d w_{i,j}}\\ &= \relu{\inner{\bw_i,\bx}+b_i}.
        \end{align*}
        Thus, we have that $\tilde{\Ncal}(\bx_{0.5})=\Max_d(\bx)$ for all $\bx_{0.5}\in[-0.5,0.5]^d$. We further modify $\tilde{\Ncal}$ by subtracting $0.5$ from the bias term of the output neuron, to obtain
        \[
            \tilde{\Ncal}(\bx_{0.5})=\Max_d(\bx)-0.5=\Max_d(\bx_{0.5})
        \]
        for all $\bx_{0.5}\in[-0.5,0.5]^d$.
    
        We now have that $\tilde{\Ncal}(0,x_1,x_2)=\max\{0,x_1,x_2\}$ for all $x_1,x_2\in[-0.5,0.5]$. By setting the incoming weights that receive the first input to zero in all of the hidden neurons, we effectively simulate the same computations that the network performs when the input is $(0,x_1,x_2)$ for $x_1,x_2\in[-0.5,0.5]$. This way, we obtain a network $\Ncal'$ satisfying $\Ncal'(x_1,x_2)=\max\{0,x_1,x_2\}$ for all $x_1,x_2\in[-0.5,0.5]$. Note that $\Ncal'$ has the same width as $\tilde{\Ncal}$, which has the same width as $\Ncal$, and thus $\Ncal'$ has width $n$.
        
        By virtue of \lemref{lem:hyperplane_arrangement}, the set of non-differentiable points of $\Ncal'$ is a hyperplane arrangement with at most $n$ hyperplanes. However, on the domain $[-0.5,0.5]^2$, the set of non-differentiable points is precisely 
        \[
            A\coloneqq\{(x_1,x_2):0\le x_1=x_2\le1\} \cup \{(x_1,0):-1\le x_1\le0\} \cup \{(0,x_2):-1\le x_2\le0\}.
        \]
        Suppose by contradiction that there exists a hyperplane arrangement whose intersection with the set $[-0.5,0.5]^2$ equals $A$. Pick any distinct $n+1$ points on the line $\{(x_1,x_2):0\le x_1=x_2\le1\}$. By the pigeonhole principle, there exists a hyperplane which intersect at least two points on this line. This implies that this hyperplane must intersect the entire line, hence it equals the set $\{(x_1,x_2):x_1=x_2\in\reals\}$. In particular, the point $(-0.25,-0.25)$ must be in this hyperplane, hence it is a non-differentiable point of $\Ncal'$. But $\Ncal'$ is constant and equals zero for all $x_1,x_2\in[-0.5,0]$, which is a contradiction.
    \end{proof}

    In the lemma below, used to prove our main ``homogenization" argument, given some $\varepsilon>0$ we let $B_{\varepsilon}(\bzero)\coloneqq\{\bx\in\reals^d:\norm{\bx}<\varepsilon\}$ denote the origin-centered Euclidean ball of radius $\varepsilon$.
    
    \begin{lemma}\label{lem:homogenized}
        Suppose that $\Ncal:\reals^d\to\reals$ is a ReLU network that satisfies $\Ncal(\bx)=\Max_d(\bx)$ for all $\bx\in[0,1]^d$. Then, there exists a ReLU network $\Ncal'$ such that
        \begin{enumerate}
            \item 
            $\Ncal'$ has the same depth and at most twice the width of $\Ncal$.
            \item 
            $\Ncal'$ is homogeneous.
            \item 
            There exists $\varepsilon>0$ such that $\Ncal'(\bx)=\Max_d(\bx)$ for all $\bx\in B_{\varepsilon}(\bzero)$.
            \item 
            All the weights of the neurons in the first hidden layer of $\tilde{\Ncal}(\bx)$ contain at least two coordinates that are non-zero.
        \end{enumerate}
    \end{lemma}
    
    \begin{proof}
        We begin with `shifting' the network $\Ncal$ to compute the maximum on a domain containing an open neighborhood of the origin. Consider the network $\tilde{\Ncal}(\cdot)$ obtained by transforming all the biases of the first hidden layer neurons using the transformation $b_i\mapsto b_i + c\sum_{j=1}^d w_{i,j}$ for some $c\in\Ucal([0.4,0.6])$, where $w_{i,j}$ is the $j^{\text{th}}$ coordinate of the weight of the $i^{\text{th}}$ neuron in the first hidden layer. For all $\bx\in[0,1]^d$, let $\bx_c\coloneqq\bx-(c,\ldots,c)$. Then we have that $\bx_c\in[-c,1-c]^d$ and that any neuron in the first hidden layer now computes for any $\bx_c$
        \[
            \relu{\inner{\bw_i,\bx_c}+b_i-c\sum_{j=1}^d w_{i,j}} = \relu{\inner{\bw_i,\bx-(c,\ldots,c)}+b_i+c\sum_{j=1}^d w_{i,j}} = \relu{\inner{\bw_i,\bx}+b_i}.
        \]
        Thus, we have that $\tilde{\Ncal}(\bx_c)=\Max_d(\bx)$ for all $\bx_c\in[-c,1-c]^d$. We further modify $\tilde{\Ncal}$ by subtracting $c$ from the bias term of the output neuron, to obtain
        \[
            \tilde{\Ncal}(\bx_c)=\Max_d(\bx)-c=\Max_d(\bx_c)
        \]
        for all $\bx_c\in[-c,1-c]^d$. Since $c\in[0.4,0.6]$, we have that
        $[-0.4,0.4]^d\subseteq[-c,1-c]^d$, so $\tilde{\Ncal}$ computes $\Max_d$ on a domain which contains an open neighborhood of the origin.
    
        Next, we gradually `homogenize' $\tilde{\Ncal}$, by moving biases in every hidden layer one layer forward. We begin with constructing a new network $\Ncal_1$ which has no biases in its first hidden layer as follows: 
        \begin{itemize}
            \item
            For any neuron in the first hidden layer with non-zero bias, if the bias is negative, we remove the neuron and set all the incoming weights of neurons which receive its output as input in the second hidden layer to zero. This way we are only left with biases that are either zero or positive, and for a sufficiently small $\varepsilon'>0$, we have that $\Ncal_1(\bx)=\Max_d(\bx)$ for all $\bx\in B_{\varepsilon'}(\bzero)$.
            \item 
            We now further modify $\Ncal_1$ to remove its positive biases in the first hidden layer. For each neuron with a positive bias $b>0$ and weight vector $\bw$, we remove the neuron and replace it with two neurons having weights $\bw$ and $-\bw$, and zero biases. For each neuron in the second hidden layer with an incoming weight $v$ from the neuron which had a positive bias, we set its incoming weights for the two new neurons to $v$ and $-v$, and modify its bias term according to $b'\mapsto b'+vb$, so that it now computes its previous output without the input from the neuron we are removing in the first hidden layer, plus $v\relu{\inner{\bw,\bx}}-v\relu{-\inner{\bw,\bx}}+vb=v(\inner{\bw,\bx}+b)$. Since $b>0$, for sufficiently small $\varepsilon''>0$, this equals $v\relu{\inner{\bw,\bx}+b}$ for all $\bx\in B_{\varepsilon''}(\bzero)$, effectively replacing the neuron in the first hidden layer. 
        \end{itemize}
        Replacing all the neurons in the first hidden layer in this manner and intersecting all the balls centered at the origin which maintain the original output of the network, we obtain $\Ncal_1$ whose first hidden layer has width at most twice of that as $\tilde{\Ncal}$, but with zero biases, and which coalesces with $\tilde{\Ncal}$ on the set $B_{\varepsilon_1}$ for some $\varepsilon_1>0$. Continuing in this manner for all the subsequent $k$ hidden layers, we obtain $\Ncal_k$ whose width is at most twice of that as $\tilde{\Ncal}$, but with zero biases in all hidden layers, and which coalesces with $\tilde{\Ncal}$ on the set $B_{\varepsilon}$ for some $\varepsilon>0$. In particular, we have that $\Ncal_k(\bzero)=\Max_d(\bzero)=0$. However, since $\Ncal_k$ has no non-zero biases in its hidden layers, we also have that $\Ncal_k(\bzero)$ equals the bias term of the output neuron, which therefore must also equal zero, proving that $\Ncal_k$ is homogeneous.
    
        It is only left to show that the neurons in the first hidden layer of $\Ncal'\coloneqq\Ncal_k$ all have at least two coordinates that are non-zero. First, we may assume without loss of generality that no neuron in the first hidden layer of $\Ncal'$ has an all-zero weight vector. This is justified, since once homogenized, this neuron has a bias of zero, and therefore it can simply be removed without changing the function computed by $\Ncal'$. Next, observe that in all stages in the construction of $\Ncal'$, the number of non-zero weights in each neuron remains unchanged. If some neuron in the first hidden layer had a weight vector $\bw$ with just a single non-zero coordinate $w_j$, then its bias term was modified to $b\mapsto b +c\cdot w_j$, which is non-zero with probability $1$. We thus have that for almost all $c\in[0.4,0.6]$, once homogenized, $\tilde{\Ncal}$ has no such neurons. Choosing one such $c$ arbitrarily and substituting its values in $\tilde{\Ncal}$, we obtain the desired $\Ncal'$.
        
    \end{proof}

    \begin{proposition}\label{prop:compact_to_unbounded}
        Suppose that $\Ncal:\reals^d\to\reals$ is a ReLU network that satisfies $\Ncal(\bx)=\Max_d(\bx)$ for all $\bx\in[0,1]^d$. Then, there exists a ReLU network $\tilde{\Ncal}:\reals^d\to\reals$ such that
        \begin{enumerate}
            \item 
            $\tilde{\Ncal}$ has the same depth and at most twice the width of $\Ncal$.
            \item 
            $\tilde{\Ncal}(\bx)=\Max_d(\bx)$ for all $\bx\in\reals^d$.
            \item 
            All the weights of the neurons in the first hidden layer of $\tilde{\Ncal}(\bx)$ contain at least two coordinates that are non-zero.
        \end{enumerate}
    \end{proposition}
    
    \begin{proof}
        By virtue of \lemref{lem:homogenized}, we obtain a network $\tilde{\Ncal}$ which satisfies the following properties:
        \begin{enumerate}
            \item 
            $\tilde{\Ncal}$ has the same depth and at most twice the width of $\Ncal$.
            \item\label{item:prop2}
            $\tilde{\Ncal}$ is homogeneous.
            \item\label{item:prop3}
            There exists $\varepsilon>0$ such that $\tilde{\Ncal}(\bx)=\Max_d(\bx)$ for all $\bx\in B_{\varepsilon}(\bzero)$.
            \item 
            All the weights of the neurons in the first hidden layer of $\tilde{\Ncal}(\bx)$ contain at least two coordinates that are non-zero.
        \end{enumerate}
        To conclude the proof, we first have that $\tilde{\Ncal}(\bzero)=\Max_d(\bzero)$ by \itemref{item:prop3}, and for all $\bx\neq\bzero$ we have
        \[
            \tilde{\Ncal}(\bx) = \frac{\norm{\bx}}{\varepsilon}\cdot\tilde{\Ncal}\p{\varepsilon\cdot\frac{\bx}{\norm{\bx}}} = \frac{\norm{\bx}}{\varepsilon}\cdot \Max_d\p{\varepsilon\cdot\frac{\bx}{\norm{\bx}}} = \Max_d(\bx),
        \]
        where the first equality follows from \itemref{item:prop2}, the second equality follows from \itemref{item:prop3}, and the last equality follows from the homogeneity of $\Max_d(\cdot)$.
    \end{proof}
    
    It will be convenient to use the following corollary of Turán's Theorem in our proofs.
    
    \begin{corollary}\label{cor:delta_turan}
        For all $\delta>0$, if $d\ge\sqrt{\frac{2}{\delta}}$ and $G$ is a graph with at least $\p{1-\frac{1-\delta}{r-1}}\frac{d^2}{2}$ edges, then $G$ contains a clique of size $r$.    
    \end{corollary}
    
    \begin{proof}
        By the assumption $d\ge\sqrt{\frac{2}{\delta}}$ we have $\frac{d^2}{2}\ge\frac{1}{\delta}$. We thus compute
        \[
            \p{1-\frac{1-\delta}{r-1}}\frac{d^2}{2} = \p{1-\frac{1}{r-1}}\frac{d^2}{2} + \delta\frac{d^2}{2} \ge \p{1-\frac{1}{r-1}}\frac{d^2}{2} + 1,
        \]
        and the corollary follows from Theorem~\ref{thm:turan}.
    \end{proof}
    
    \subsection{Proof of Theorem~\ref{thm:depth3}}\label{app:depth3_proof}
        First, it is straightforward to verify that for all $d\in\{1,2,3,4\}$,
        \[
            \floor{\p{\frac18-\frac{1}{4d}-\frac{1}{2d^2}}d^2} \le 0,
        \]
        and therefore the theorem statement holds vacuously for such values of $d$. Assume from now on that $d\ge5$, and denote $a_d\coloneqq\frac18-\frac{1}{4d}-\frac{1}{2d^2}$.
        
        Suppose that $\Ncal$ is a depth-3 ReLU network that computes $\Max_d$ on $[0,1]^d$ with width at most $a_dd^2$. Then by \propref{prop:compact_to_unbounded}, there exists a depth-3 ReLU network $\tilde{\Ncal}$ of width at most $2a_dd^2$ such that $\Ncal(\bx)=\Max_d(\bx)$ for all $\bx\in\reals^d$. Moreover, we have that all the weight vectors $\bw_i$ of the neurons in the first hidden layer of $\tilde{\Ncal}$ contain at least two non-zero coordinates. Let $G_{\tilde{\Ncal}}$ be the corresponding graph induced by the first hidden layer of $\tilde{\Ncal}$. Then by the properties of $\tilde{\Ncal}$, $G_{\tilde{\Ncal}}$ has at least
        \begin{align*}
            \binom{d}{2} - 2a_dd^2 &= \frac{d^2}{2} - \frac{d}{2} - 2\p{\frac18 - \frac{1}{4d} - \frac{1}{2d^2}}d^2 = \frac{d^2}{2} - \frac{d}{2} -\frac{d^2}{4} + \frac{d}{2} + 1 = \frac{d^2}{4}+1
        \end{align*}
        edges. By Theorem~\ref{thm:turan}, $G_{\tilde{\Ncal}}$ contains a clique of size $3$. By definition, all the neurons in the first hidden layer of $\tilde{\Ncal}$ have a non-zero coordinate with an index that does not form this clique. This holds true since otherwise, if all the coordinates of a particular neuron's weight vector are zero at the indices that do not form the clique, then since it must have at least two non-zero coordinates, we have that these two coordinates are both at indices that form the clique. But this implies that these two coordinates remove an edge from the clique, which is a contradiction. 
    
        Suppose without loss of generality that the clique is formed by the first three coordinates. Then we have shown that all neurons have at least one weight coordinate with index $4$ or larger. We now define
            \[
                w_{\min}\coloneqq\min_{i,j:w_{i,j}\neq0}|w_{i,j}|,\hskip 0.3cm w_{\max}\coloneqq\max_{i,j}|w_{i,j}|, \hskip 0.3cm W\coloneqq \frac{w_{\max}}{w_{\min}}.
            \]
            That is, $W$ is the maximal ratio between two coordinates of weights in the first hidden layer of $\Ncal$. Consider the function computed by the following assignment of inputs for the coordinates $4,\ldots,d$. For any such coordinate with index $i$, we substitute $x_i=-3\cdot(2W)^{i-3}$ as follows
            \[
                f(x_1,x_2,x_3)\coloneqq\Ncal(x_1,x_2,x_3,x_4=-6W,x_5=-12W^2,\ldots,x_d=-3\cdot2^{d-3}W^{d-3}).
            \]
            Since $\tilde{\Ncal}$ computes $\Max_d(\bx)$ for all $\bx\in\reals^d$ and since the assignment assigns negative values for all coordinates with index $4$ or larger, we have that $f(x_1,x_2,x_3)=\max\{x_1,x_2,x_3\}$ for all $\bx\in[0,1]^3$. Let $\bw=(w_1,\ldots,w_d)$ denote the weights of an arbitrary neuron and let $\ell\ge4$ be the largest index of a coordinate such that $w_\ell\neq0$. Then we have for all $\bx\in[0,1]^3$ that
            \begin{align}
                \abs{\sum_{i=1}^{\ell-1}x_iw_i} &= \abs{\sum_{i=1}^3x_iw_i+\sum_{i=4}^{\ell-1}-3\cdot2^{i-3}W^{i-3}w_i} \le \sum_{i=1}^3\abs{x_iw_i}+3\sum_{i=4}^{\ell-1}2^{i-3}W^{i-3}\abs{w_i} \nonumber\\
                &
                \le 3w_{\max} + 3w_{\max}\sum_{i=4}^{\ell-1}(2W)^{i-3} = 3w_{\max} + 3w_{\max}\sum_{i=1}^{\ell-4}(2W)^{i} \nonumber\\
                &= 3w_{\max}\sum_{i=0}^{\ell-4}(2W)^{i} = 
                3w_{\max}\frac{(2W)^{\ell-3}-1}{2W-1} < 3w_{\max}\frac{(2W)^{\ell-3}}{W} \nonumber\\
                &=3w_{\min}(2W)^{\ell-3} \le 3\abs{w_{\ell}}(2W)^{\ell-3} = |x_{\ell}w_{\ell}|.\label{eq:ell_dominates}
            \end{align}
            In the above, the penultimate inequality follows from $(2W)^{\ell-3}-1<(2W)^{\ell-3}$ and $2W-1\ge W$ which holds since $W\ge1$, the penultimate equality and last inequality follow from the definitions of $w_{\min}$, $w_{\max}$ and $W$, and the final equality follows from $x_{\ell}=-3\cdot(2W)^{i-3}$. We thus have that for all $\bx\in[0,1]^3$, 
            \[
                \sign\p{\sum_{i=1}^dx_iw_i} = \sign\p{\sum_{i=1}^{\ell-1}x_iw_i + x_\ell w_\ell} = \sign\p{x_\ell w_\ell} = -\sign\p{w_\ell},
            \]
            where the first equality follows from the fact that $\ell$ is the largest non-zero coordinate of $\bw$, the second equality follows from Equation~(\ref{eq:ell_dominates}), and the last equality follows from $x_{\ell}<0$.
            We therefore obtain that the sign of the pre-activation output $\sum_{i=1}^dx_iw_i$ of each neuron is fixed for all $\bx\in[0,1]^3$ and determined by $w_{\ell}$. 
            
            To conclude the derivation thus far, we have established that there exists a neural network which computes $\Max_3$ on $\bx\in[0,1]^3$, where no neuron in the first hidden layer changes its linearity over this domain. This implies that we can remove the ReLU non-linearities in the first hidden layer as follows: 
            \begin{itemize}
                \item
                If a neuron's pre-activation output is negative, then it is negative for all $\bx\in[0,1]^3$, and thus this neuron always outputs zero. By changing the incoming weight of all the neurons in the second hidden layer to zero at the coordinate which receives it as input, we can simulate it without the ReLU activation. Note that this preserves the architecture without impacting the number of neurons.
                \item 
                If a neuron's pre-activation output is positive, then it is positive for all $\bx\in[0,1]^3$. In such a case, each neuron merely applies a linear transformation to its input and propagates it forward to the second hidden layer. By composing this linear transformation with the linear transformation computed by the second hidden layer, we are able to absorb it in the weights of the second hidden layer. Note that this also does not impact the architecture of the network, and maintains the same width while removing all non-linearities in the first hidden layer.
            \end{itemize}
            Finally, we conclude that in both cases, we can remove the non-linearities of the first hidden layer while maintaining the same behavior over the domain $[0,1]^3$. Thus, we can collapse the first hidden layer which is now linear, and obtain a depth-2 ReLU network that computes $\Max_3$ over the domain $[0,1]^3$. But since this contradicts \propref{prop:max3_depth2} which establishes that no depth-2 network can compute this function, the theorem then follows.
    
        \subsection{Proof of Theorem~\ref{thm:depthk}}\label{app:main_proof}
    
        We first define $\alpha_k\coloneqq\frac{1}{2^{k-2}-1}$ for all $k\ge3$. We now have
        \begin{align}
            \p{1-\alpha_{k}}\p{1+\alpha_{k-1}} &= \p{1-\frac{1}{2^{k-2}-1}}\p{1+\frac{1}{2^{k-3}-1}} = \frac{2^{k-2}-2}{2^{k-2}-1} \cdot \frac{2^{k-3}}{2^{k-3}-1}\nonumber\\
            &= \frac{2^{k-2}}{2^{k-2}-1} = 1 + \frac{1}{2^{k-2}-1} = 1+\alpha_{k}.\label{eq:alpha_property}
        \end{align}
        Next, we note that our assumption $3\le\log_2(\log_2(d))$ implies that
        \begin{equation}\label{eq:d_lb}
            d\ge2^{2^3}=256,
        \end{equation}
        and $k\le\log_2(\log_2(d))$ implies that
        \[
            d^{\alpha_{k}} \ge d^{\frac{1}{2^{\log_2(\log_2(d))-2}-1}} = d^{\frac{4}{2^{\log_2(\log_2(d))}-4}} = d^{\frac{4}{\log_2(d)-4}} = 2^{\frac{4\log_2(d)}{\log_2(d)-4}}=16\cdot2^{\frac{16}{\log_2(d)-4}}.
        \]
        Since $2^{\frac{16}{\log_2(d)-4}}>1$ for all $d\ge256$, we have by Equation~(\ref{eq:d_lb}) that
        \begin{equation}\label{eq:d_exp_k_lb}
            d^{\alpha_{k}} \ge 16.
        \end{equation}
        
        The proof of the theorem now follows by induction. For the base case $k=3$, it is straightforward to verify that for all $d\ge256$ we have
        \[
            \frac18-\frac{1}{4d}-\frac{1}{2d^2} \ge 0.1.
        \]
        Thus, Equation~(\ref{eq:d_lb}) and Theorem~\ref{thm:depth3} imply that a depth-3 ReLU network that computes $\Max_3$ on $[0,1]^3$ has width at least $0.1d^2$. On the other hand, the theorem statement for the case $k=3$ is that width at least $0.1d^{1+\frac{1}{2^{3-2}-1}}=0.1d^2$ is required, which concludes the proof of the base case.
    
        For the inductive step, suppose by contradiction that $\Ncal_{k}$ is a depth-$k$ ReLU network of width at most $0.1d^{1+\alpha_{k}}$ that satisfies $\Ncal_{k}(\bx)=\Max_d$ for all $\bx\in[0,1]^d$. By \propref{prop:compact_to_unbounded}, there exists a depth-$k$ ReLU network $\tilde{\Ncal}_{k}$ of width at most $0.2d^{1+\alpha_{k}}$ that satisfies $\tilde{\Ncal}_{k}(\bx)=\Max_d(\bx)$ for all $\bx\in\reals^d$. Moreover, all the weights of the neurons in the first hidden layer of $\tilde{\Ncal}_{k}$ contain at least two coordinates that are non-zero.

        Let $G_{\tilde{\Ncal}_{k}}$ be the corresponding graph induced by the first hidden layer of $\tilde{\Ncal}_{k}$, as defined in Step~3 in Subsection~\ref{subsec:step3}. Then by the width upper bound on $\tilde{\Ncal}_{k}$ and the definition of $G_{\tilde{\Ncal}_{k}}$, we have that $G_{\tilde{\Ncal}_{k}}$ has at least
        \[
            \binom{d}{2} - \frac{1}{5}d^{1+\alpha_{k}} = \frac{d^2}{2}-\frac{d}{2} - \frac{1}{5}d^{1+\alpha_{k}} = \frac{d^2}{2} - d\p{\frac{1}{2} + \frac{1}{5}d^{\alpha_{k}}}
        \]
        edges.
        By Equation~(\ref{eq:d_exp_k_lb}), we have that $\frac12\le\frac{1}{32}d^{\alpha_{k}}$, implying that the above displayed equation is lower bounded by
        \[
            \frac{d^2}{2} - d\p{\frac{1}{32} d^{\alpha_{k}} + \frac{1}{5}d^{\alpha_{k}}} = \frac{d^2}{2}\p{1-\frac{37}{80d^{1-\alpha_{k}}}} = \frac{d^2}{2}\p{1-\frac{\frac{37}{38}}{\frac{80}{38}d^{1-\alpha_{k}}}} \ge \frac{d^2}{2}\p{1-\frac{1-\frac{1}{38}}{2.1d^{1-\alpha_{k}}}},
        \]
        where the inequality holds since $\frac{80}{38}\ge2.1$. Let $\delta=\frac{1}{38}$. Then by Equation~(\ref{eq:d_lb}), we have $d\ge256\ge\sqrt{76}=\sqrt{\frac{2}{\delta}}$, and therefore from the above displayed equation and by virtue of \corollaryref{cor:delta_turan}, we deduce that $G_{\tilde{\Ncal}_{k}}$ contains a clique of size
        \begin{equation}\label{eq:r_lb}
            r=\floor{2.1d^{1-\alpha_{k}}+1} \ge 2.1d^{1-\alpha_{k}}.
        \end{equation}
        By definition, all the neurons in the first hidden layer of $\tilde{\Ncal}_{k}$ have a non-zero coordinate with an index that does not form this clique. This holds true since otherwise, if all the coordinates of a particular neuron's weight vector are zero at the indices that do not form the clique, then since it must have at least two non-zero coordinates, we have that these two coordinates are both at indices that form the clique. But this implies that these two coordinates remove an edge from the clique, which is a contradiction.
    
        Suppose without loss of generality that the clique is formed by the first $r$ coordinates. Then we have shown that all neurons have at least one non-zero weight coordinate with index $r+1$ or larger. We now define
            \[
                w_{\min}\coloneqq\min_{i,j:w_{i,j}\neq0}|w_{i,j}|,\hskip 0.3cm w_{\max}\coloneqq\max_{i,j}|w_{i,j}|, \hskip 0.3cm W\coloneqq \frac{w_{\max}}{w_{\min}}.
            \]
            That is, $W$ is the maximal ratio between two coordinates of weights in the first hidden layer of $\Ncal$. Consider the function computed by the following assignment of inputs for the coordinates $r+1,\ldots,d$. For any such coordinate with index $i$, we substitute $x_i=-r\cdot(2W)^{i-r}$ as follows
            \[
                f(x_1,\ldots,x_r)\coloneqq\Ncal(x_1,\ldots,x_r,x_{r+1}=-2rW,x_{r+2}=-4rW^2,\ldots,x_d=-r\cdot2^{d-r}W^{d-r}).
            \]
            Since $\tilde{\Ncal}$ computes $\Max_d(\bx)$ for all $\bx\in\reals^d$ and since the assignment assigns negative values for all coordinates with index $r+1$ or larger, we have that $f(x_1,\ldots,x_r)=\max\{x_1,\ldots,x_r\}$ for all $\bx\in[0,1]^r$. Let $\bw=(w_1,\ldots,w_d)$ denote the weights of an arbitrary neuron and let $\ell\ge r+1$ be the largest index of a coordinate such that $w_\ell\neq0$. Then we have for all $\bx\in[0,1]^3$ that
            \begin{align}
                \abs{\sum_{i=1}^{\ell-1}x_iw_i} &= \abs{\sum_{i=1}^rx_iw_i+\sum_{i=r+1}^{\ell-1}-r\cdot2^{i-r}W^{i-r}w_i} \le \sum_{i=1}^r\abs{x_iw_i}+r\sum_{i=r+1}^{\ell-1}2^{i-r}W^{i-r}\abs{w_i} \nonumber\\
                &
                \le rw_{\max} + rw_{\max}\sum_{i=r+1}^{\ell-1}(2W)^{i-r} = rw_{\max} + rw_{\max}\sum_{i=1}^{\ell-r-1}(2W)^{i} \nonumber\\
                &= rw_{\max}\sum_{i=0}^{\ell-r-1}(2W)^{i} = 
                rw_{\max}\frac{(2W)^{\ell-r}-1}{2W-1} < rw_{\max}\frac{(2W)^{\ell-r}}{W} \nonumber\\
                &=rw_{\min}(2W)^{\ell-r} \le r\abs{w_{\ell}}(2W)^{\ell-r} = |x_{\ell}w_{\ell}|.\label{eq:ell_dominates2}
            \end{align}
            In the above, the penultimate inequality follows from $(2W)^{\ell-r}-1<(2W)^{\ell-r}$ and $2W-1\ge W$ which holds since $W\ge1$, the penultimate equality and last inequality follow from the definitions of $w_{\min}$, $w_{\max}$ and $W$, and the final equality follows from $x_{\ell}=-r\cdot(2W)^{i-r}$. We thus have that for all $\bx\in[0,1]^r$, 
            \[
                \sign\p{\sum_{i=1}^dx_iw_i} = \sign\p{\sum_{i=1}^{\ell-1}x_iw_i + x_\ell w_\ell} = \sign\p{x_\ell w_\ell} = -\sign\p{w_\ell},
            \]
            where the first equality follows from the fact that $\ell$ is the largest non-zero coordinate of $\bw$, the second equality follows from Equation~(\ref{eq:ell_dominates2}), and the last equality follows from $x_{\ell}<0$.
            We therefore obtain that the sign of the pre-activation output $\sum_{i=1}^dx_iw_i$ of each neuron is fixed for all $\bx\in[0,1]^r$ and determined by $w_{\ell}$. 
            
            To conclude the derivation thus far, we have established that there exists a neural network which computes $\Max_r$ on $\bx\in[0,1]^r$, where no neuron in the first hidden layer changes its linearity over this domain. This implies that we can remove the ReLU non-linearities in the first hidden layer as follows: 
            \begin{itemize}
                \item
                If a neuron's pre-activation output is negative, then it is negative for all $\bx\in[0,1]^r$, and thus this neuron always outputs zero. By changing the incoming weight of all the neurons in the second hidden layer to zero at the coordinate which receives it as input, we can simulate it without the ReLU activation. Note that this preserves the architecture without impacting the number of neurons.
                \item 
                If a neuron's pre-activation output is positive, then it is positive for all $\bx\in[0,1]^r$. In such a case, each neuron merely applies a linear transformation to its input and propagates it forward to the second hidden layer. By composing this linear transformation with the linear transformation computed by the second hidden layer, we are able to absorb it in the weights of the second hidden layer. Note that this also does not impact the architecture of the network, and maintains the same width while removing all non-linearities in the first hidden layer.
            \end{itemize}
            Finally, we conclude that in both cases, we can remove the non-linearities of the first hidden layer while maintaining the same behavior over the domain $[0,1]^r$. Thus, we can collapse the first hidden layer which is now linear, and obtain a depth-$(k-1)$ ReLU network $\Ncal_{k-1}$ with width at most $0.2d^{1+\alpha_{k}}$ that computes $\Max_r$ over the domain $[0,1]^r$.
    
            We will now derive a contradiction to the induction hypothesis, which implies that $\Ncal_{k-1}$ must have width at least $0.1r^{1+\alpha_{k-1}}$, meaning that it must hold that $0.2d^{1+\alpha_{k}} \ge 0.1r^{1+\alpha_{k-1}}$. We compute
            \begin{align*}
                0.2d^{1+\alpha_{k}} &\ge 0.1r^{1+\alpha_{k-1}} \ge 0.1\p{2.1d^{1-\alpha_k}}^{1+\alpha_{k-1}} = 0.1\cdot2.1^{1+\alpha_{k-1}}d^{1+\alpha_k} \ge 0.1\cdot2.1d^{1+\alpha_k}\\
                &= 0.2d^{1+\alpha_k} + 0.01d^{1+\alpha_k} \ge 0.2d^{1+\alpha_k} + 0.01d \ge 0.2d^{1+\alpha_k} + 2.  
            \end{align*}
            In the above, the second inequality follows from Equation~(\ref{eq:r_lb}), the first equality follows from Equation~(\ref{eq:alpha_property}), the third and fourth inequalities follow from $\alpha_k\ge0$ for all $k$, and the last inequality follows from Equation~(\ref{eq:d_lb}). We have thus reached a contradiction, concluding the proof of the theorem.
    
\end{document}